\documentclass[11pt]{article}

\usepackage{color}
\usepackage{Definitions}

\newcommand{\hEE}{\hat{\EE}}
\newcommand{\con}{\mathrm{con}}
\newcommand{\Lcal}{\mathcal{L}}
\newcommand{\Kcal}{\mathcal{K}}
\newcommand{\est}{{\rm est}}

\begin{document}

\title{A Note on Improved Loss Bounds for Multiple Kernel Learning}

\author{Zakria Hussain \& John Shawe-Taylor \\
       Department of Computer Science\\
       University College London\\
       London, WC1E 6BT, UK \\
       e-mail: \texttt{\{z.hussain,jst\}@cs.ucl.ac.uk}\\
       \ \\
	Mario Marchand\\
        Computer Science and Software Engineering \\
           Laval University, Qu\'ebec (QC), Canada \\
            e-mail: \texttt{mario.marchand@ift.ulaval.ca} 
}


\maketitle

\begin{abstract}
In this paper, we correct an upper bound, presented in~\cite{hs-11}, 
on the generalisation error of classifiers learned through multiple kernel learning. The bound in~\cite{hs-11} uses Rademacher complexity and has an\emph{additive} dependence on the logarithm of the number of kernels and the margin achieved by the classifier. However, there are some errors in parts of the proof which are corrected in this paper. Unfortunately, the final result turns out to be a risk bound which has a \emph{multiplicative} dependence on the logarithm of the number of kernels and the margin achieved by the classifier. 
\end{abstract}


\section{Introduction}

We refer to \cite{hs-11} for the motivation and definitions of multiple kernel learning. It presents a number of results, including a new Rademacher complexity bound on the generalisation error of classifiers learned from a multiple kernel class with a logarithmic dependence on the number of kernels used and with that logarithm entering additively into the bound---that is, independently of the complexity of the individual kernels or the margin of the classifier on the training set.

In this paper, we follow the approach presented in~\cite{hs-11} but correct some of the errors that are present. Unfortunately, the Rademacher complexity risk bound turns out to exhibit a \emph{multiplicative} dependence on the logarithm of the number of kernels and the margin achieved by the classifier. 

\section{Detailed proof}
\label{sec:analysis}


\subsection{Preliminaries}
\label{sec:pre}

Let $\z = \{(x_i,y_i)\}^m_{i=1}$ be an $m$-sample where $x_i \in \Xcal \subset \Reals^n$ and $y_i \in \Ycal = \{-1,+1\}$, with $\mathcal{Z} = \Xcal \times \Ycal$.  Let $\x = \{x_1,\ldots, x_m\}$ contain the input vectors.


\begin{definition}[\cite{abr-64}]
A kernel is a function $\kappa$ that for all $x,x' \in \Xcal$ satisfies
\[
\kappa \LP x,x' \RP = \langle \phi(x),\phi(x') \rangle,
\]
where $\phi$ is a mapping from $\Xcal$ to an (inner product) Hilbert space $\Hcal$
\[
\phi : \Xcal \mapsto \Hcal.
\]

\end{definition}

Kernel learning algorithms~\cite{ss-02,stc-04} make use of the $m \times m$ kernel matrix $K = [ \kappa(x_i,x_{i'}) ]_{i,{i'}=1}^m$ defined using the training inputs $\x$.  When using the kernel representation it is not always possible to represent the weight vector $w$ explicitly and so we can use the function $f$ directly as the predictor:
\[
f(x) = \sum_{i=1}^{m} \alpha_i y_i\kappa(x_i,x) = \langle w, \phi(x)\rangle,
\]
where $\alpha = (\alpha_1,\ldots, \alpha_m)$ is the dual weight vector and the corresponding norm of the weight vector is
\[
\|w\|^2  = \sum_{i,j=1}^{m} \alpha_iy_i\alpha_jy_j \kappa(x_i,x_j).
\]
Given a kernel $\kappa$, we will use $\phi_\kappa(\cdot)$ to denote a feature space mapping satisfying
\[
\kappa(x,x') = \langle \phi_\kappa(x),\phi_\kappa(x') \rangle.
\]
Hence, learning with a kernel $\kappa$ can be described as finding a function from the class of functions \cite{sbd-06} 
\[
\Fcal_{\kappa} = \LC x \mapsto \langle w,\phi_\kappa(x) \rangle \ \left | \right. \; \| w \|_2 \leq 1, \RC
\]
minimising the empirical average of the hinge loss
\[
h^\gm (yf(x)) = \max\LP1 - \frac{y f(x)}{\gm},0\RP.
\]
where we call $\gm \in [0,1]$ the \emph{margin}. For multiple kernel learning we consider a family of kernels $\Kcal$ and the corresponding function class
\[
\Fcal_{\Kcal} = \LC x \mapsto \langle w,\phi_\kappa(x) \rangle \ \left | \right. \; \| w \|_2 \leq 1, \mbox{ for some } \kappa \in \Kcal \RC\, .
\]

For a distribution $\Dcal$, we use the notation $\EE_\Dcal[f(x)]$ to denote the expected value of $f(x)$ when $x \sim \Dcal$. Given a training set $\x$ we denote $\hEE[f]$ to denote its empirical average over the sample $\x$.

For the generalisation error bounds we assume that the data are generated iid from a fixed but unknown probability distribution $\Dcal$ over the joint space $\Xcal \times \Ycal$.  Given the \emph{true error} of a function $f$:
\[
\err(f) = \EE_{(x,y) \sim \Dcal}(y f(x) \leq 0) = \EE_\Dcal[yf(x)],
\]
the \emph{empirical margin error} of $f$ with margin $\gm>0$:
\begin{eqnarray*}
\herr^{\gm}(f) = \frac{1}{m} \sum_{i=1}^{m} \mathbb{I} ( y_i f(x_i) < \gm ) = \hEE[\mathbb{I} ( y_i f(x_i) < \gm )]\, ,
\end{eqnarray*}
where $\mathbb{I}$ is the indicator function, and the estimation error $\est^{\gm}(f)$ is defined as
\[
 \est^{\gm}(f) = | \err(f) - \herr^{\gm}(f) |,
\]
we would like to find an upper bound for $\est^{\gm}(f)$.  
In the sequel we will state the bounds in standard form, where the true error $\err(f)$ of a function $f$ is upper bounded by the empirical margin error $\herr^\gm(f)$ plus the estimation error $\est^\gm(f)$:
\begin{eqnarray}\label{bound:form}
\err(f) \leq \herr^\gm(f) + \est^\gm(f).
\end{eqnarray}
We further consider the clipped hinge function:
\[
\Acal^\gm(s) = \left\{\begin{array}{ll}
             0;& \mbox{if } s \geq \gm \\
             1 - s/\gm; & \mbox{if } 0 \leq s \leq \gm;\\
             1;& \mbox{otherwise},
\end{array} \right.
\]
and its empirical estimation $\hEE[\Acal^\gm(yf(x))]$. Note that
$\err(f) \leq \EE_\Dcal[\Acal^\gm(yf(x))]$, $\hEE[\Acal^\gm(yf(x))] \leq \herr^\gm(f)$ and $\hEE[\Acal^\gm(yf(x))] \leq \hEE[h^\gm(yf(x)))$.

Let $\Kcal = \{\kappa_1,\ldots,\kappa_p\}$ denote a family of kernels, where each kernel $\kappa_j$ is called the $j$th \emph{base} kernel.   The following kernel family is formed using a convex combination of base kernels:
\begin{eqnarray*}
\Kcal_{\con}(\kappa_1,\ldots,\kappa_p) = \LC \kappa^{\ld} = \sum_{j=1}^{p} \ld_j \kappa_j \left | \right. \ld_j \geq 0, \sum_{j=1}^{p} \ld_j = 1 \RC. \\
\end{eqnarray*}
Note, $p$ is the complexity of the kernel family (\ie, cardinality of the set of base kernels).

\subsection{Rademacher complexity bound for MKL}
\label{sec:rad}

In this section we correct the MKL risk bound of \cite{hs-11}. 
We begin by the following definition of Rademacher complexity.

\begin{definition}[Rademacher complexity]
For a sample $\x = \{x_1,\ldots,x_m\}$ generated by a distribution $\Dcal_{\Xcal}$ on a set $\Xcal$ and a real-valued function class $\Fcal$ with domain $\Xcal$, the \emph{empirical Rademacher complexity} of $\Fcal$ is the random variable
\begin{eqnarray*}
\hat{R}_m(\Fcal) = \EE_{\sg} \LB \sup_{f \in \Fcal} \frac{2}{m} \sum_{i=1}^{m} \sg_i f(x_i) \left| \; x_1,\ldots,x_m \right. \RB.
\end{eqnarray*}
where $\sg = (\sg_1,\ldots,\sg_m)$ are independent uniform $\{\pm 1\}$-valued (Rademacher) random variables.
The \emph{(true) Rademacher complexity} is:
\begin{eqnarray*}
R_m(\Fcal) = \EE_\x \LB \hat{R}_m(\Fcal) \RB = \EE_{\x \sg} \LB \sup_{f \in \Fcal} \frac{2}{m} \sum_{i=1}^{m} \sg_i f(x_i) \RB.
\end{eqnarray*}
\end{definition}

The standard Rademacher bound for function classes is given in the following theorem.
\begin{theorem}[\cite{bm-02}]\label{thm:main_rad}
Fix $\dt \in (0,1)$, and let $\Fcal$ be a class of functions mapping from $\mathcal{Z} = \Xcal \times \Ycal$ to $[0,1]$.  Let $\z = \{z_i\}_{i=1}^{m}$ be drawn independently according to a probability distribution $\Dcal$.  Then with probability $1-\dt$ over random draws of samples of size $m$, every $f \in \Fcal$ satisfies
\begin{eqnarray*}
\EE_\Dcal(f) & \leq & \hEE(f) + \hat{R}_{m}(\Fcal) + 3\sqrt{\frac{\ln(2/\dt)}{2m}}\, .
\end{eqnarray*}
\end{theorem}

We have attributed this bound to \cite{bm-02}, though, strictly speaking, they used the slightly weaker version of Rademacher complexity including an absolute value of the sum. This version is obtained by a slight tightening of the argument.
This bound is quite general and applicable to various learning algorithms if a tight upper bound of \emph{empirical Rademacher complexity} $\hat{R}_m(\Fcal)$ of the function class $\Fcal$ can be found.  For kernel methods, a well-known result uses the trace of the kernel matrix to bound the empirical Rademacher complexity.

\begin{theorem}[\cite{bm-02}]\label{thm:kern_rad}
If $\kappa : \Xcal \times \Xcal \mapsto \Reals$ is a kernel, and $\x = \{x_1,\ldots,x_m\}$ is a sample of points from $\Xcal$, then the empirical Rademacher complexity of the class $\Fcal_\kappa$ satisfies
\begin{eqnarray*}
\hat{R}_{m}(\Fcal_\kappa) \leq \frac{2}{m} \sqrt{\sum_{i=1}^{m} \kappa(x_i,x_i)}. 
\end{eqnarray*}
Furthermore, if $R^2 \geq \kappa(x,x)$ for all $x \in \Xcal$ and $\kappa$ is a normalised kernel such that $\sum_{i=1}^{m} \kappa(x_i,x_i) = m$, then we have
\begin{eqnarray*}
\frac{2}{m} \sqrt{\sum_{i=1}^{m} \kappa(x_i,x_i)} \leq  \frac{2R}{\sqrt{m}}.
\end{eqnarray*}
\end{theorem}

The problem of learning kernels from a convex combination of base kernels is related to using the convex hull of a set of functions. Consider
{\small
\begin{eqnarray}\label{eq:conv}
\con(\Fcal) = \LC \sum_j a_j f_j \left | \right. f_j \in \Fcal, a_j \geq 0, \sum_j a_j \leq 1 \RC.
\end{eqnarray}}

Since adding kernels corresponds to concatenating feature spaces, it is clear that (here $w_j$ is the restriction of $w$ to the feature space defined by the mapping $\phi_{\kappa_j}(\cdot)$ corresponding to kernel $\kappa_j$)
\begin{eqnarray}
\nonumber
\Fcal_{\Kcal_{\con}(\kappa_1,\ldots, \kappa_p)} &=& \LC x \mapsto \langle w,\phi_\kappa(x) \rangle \ \left | \right. \; \| w \|_2 \leq 1, \kappa = \sum_{j=1}^p \lambda_j \kappa_j, \sum_{j=1}^p \lambda_j = 1 \RC \\
\nonumber
&=& \LC x \mapsto \sum_{j=1}^p \sqrt{\lambda_j}\|w_j\| \left\langle \frac{w_j}{\|w_j\|},\phi_{\kappa_j}(x) \right\rangle\ \left | \right. \; \| w \|_2 \leq 1, \sum_{j=1}^p \lambda_j = 1   \; \RC\\
&=& \con\left(\bigcup_{j=1}^p\Fcal_{\kappa_j}\right), \label{convrelation}
\end{eqnarray}
since, by the Cauchy Schwartz inequality, we have
\[
\sum_{j=1}^p \sqrt{\lambda_j}\|w_j\| \leq \sqrt{\sum_{j=1}^p \lambda_j}\sqrt{\sum_{j=1}^p \|w_j\|^2} \leq 1.
\]
Hence, we are interested in the empirical Rademacher complexity of a convex hull as given by Equation (\ref{eq:conv}), which is well known to satisfy
\begin{eqnarray}\label{eq:con}
\hat{R}_{m}(\con(\Fcal)) = \hat{R}_{m}(\Fcal)\, .
\end{eqnarray}
Furthermore, following~\cite{kt-08} and~\cite{ast-04},  we have the following result. 
\begin{theorem}[\cite{kt-08}]\label{thm:lip}
The empirical Rademacher complexity of the function class $\Lcal(\Fcal)$ where $\Lcal(\cdot)$ is Lipschitz function with Lipschitz constant $L$ is bounded by
\begin{eqnarray*}
\hat{R}_{m}(\Lcal(\Fcal)) \leq L\hat{R}_{m}(\Fcal).
\end{eqnarray*}

\end{theorem}

Given all these results, we are now in a position to state the following theorem, which proves a high probability upper bound for the empirical Rademacher complexity of a union of function classes $\bigcup_{j=1}^p \Fcal_j =  \Fcal$.

\begin{theorem}\label{thm:boost_rad}
Let $\x = \{x_1,\ldots,x_m\}$ be an $m$-sample of points from $\Xcal$, then the empirical Rademacher complexity $\hat{R}_m$ of the class $\Fcal = \cup_{j=1}^p \Fcal_j$, where the range of all the functions in $\Fcal$ is $[0,1]$, satisfies:
\begin{eqnarray*}
\hat{R}_m(\Fcal) \leq \max_{1\leq j \leq p} \hat{R}_m(\Fcal_j) + \sqrt{\frac{8\ln (p)}{m}}.
\end{eqnarray*}
\end{theorem}
\begin{proof}
\newcommand{\hdelta}{\hat{\delta}}

Since $\Fcal$ is the union of $p$ function classes, we have
$$
\hat{R}_m(\Fcal)\ =\ \EE_{\sg} \max_{1\leq j \leq p} \sup_{f \in \Fcal_{j}} \frac{2}{m}\sum_{i=1}^m \sg_i f(x_i)\, .
$$
From Jensen's inequality, we have, for any $\ld \ge 0$, that
\begin{eqnarray}\label{eq:exp}
\nonumber\exp\LP{\ld\hat{R}_m(\Fcal)}\RP &\le& \EE_{\sg}\exp\LP{\ld\LB \max_{1\leq j \leq p} \sup_{f \in \Fcal_{j}} \frac{2}{m}\sum_{i=1}^m \sg_i f(x_i)\RB}\RP\\
\nonumber&=&  \EE_{\sg} \max_{1\leq j \leq p}\exp\LP{\ld\LB \sup_{f \in \Fcal_{j}} \frac{2}{m}\sum_{i=1}^m \sg_i f(x_i)\RB}\RP\\
&\le& \sum_{j=1}^p \EE_{\sg} \exp\LP{\ld\LB \sup_{f \in \Fcal_{j}} \frac{2}{m}\sum_{i=1}^m \sg_i f(x_i)\RB}\RP\, .
\end{eqnarray}
Now, for any fixed function class $\Fcal_j$ and any fixed training sample, let
$$\xi(\sg_1,\dots,\sg_m)\ \eqdef\ \sup_{f \in \Fcal_{j}} \frac{2}{m}\sum_{i=1}^m \sg_i f(x_i)\, .
$$
A basic result of McDiarmid~\cite{m-89} states that for any $\ld\ge0$, we have
$$
\EE e^{\ld\xi}\ \le\ e^{\frac{\ld^2}{8}\sum_{i=1}^m c_i^2}\cdot e^{\ld\EE\xi}\, ,
$$
where, for all $i$, we have
$$
\sup_{\sg_1,\dots,\sg_m,\hat{\sg}_i}\left| \xi(\sg_1,\dots,\sg_m) - \xi(\sg_1,\dots,\sg_{i-1},\hat{\sg}_i,\sg_{i+1},\dots,\sg_m)\right|\ \le\ c_i\, .
$$
In our case, we have that $c_i\leq 4/m\ \forall i\in\{1,\dots,m\}$. Hence, from Equation~\eqref{eq:exp}, we have
\begin{eqnarray*}
\exp\LP{\ld\hat{R}_m(\Fcal)}\RP &\le& e^{2\ld^2/m}\sum_{j=1}^p e^{\ld\hat{R}_m(\Fcal_j)}\, .
\end{eqnarray*}
By taking the logarithm on both sides of this equation, we obtain
\begin{eqnarray*}
\ld\hat{R}_m(\Fcal)&\le& \frac{2\ld^2}{m} + \ln\LB \sum_{j=1}^p e^{\ld\hat{R}_m(\Fcal_j)}\RB\\
&\le&  \frac{2\ld^2}{m} + \ln\LB p\cdot\max_{1\le j\le p} e^{\ld\hat{R}_m(\Fcal_j)}\RB\\
&\le& \frac{2\ld^2}{m} + \ln(p) + \max_{1\le j\le p}\ld\hat{R}_m(\Fcal_j)\, .
\end{eqnarray*}
Hence, we have
$$
\hat{R}_m(\Fcal)\ \le\  \frac{2\ld}{m} + \frac{1}{\ld}\ln(p) + \max_{1\le j\le p}\hat{R}_m(\Fcal_j)\, .
$$
The theorem then follows from this equation by choosing 
$$
\ld\ =\ \sqrt{\frac{m}{2}\ln p}\, .
$$ 
\end{proof}

Recall the function $\Acal^\gm(\cdot)$ and the properties $\err(f) \leq \EE_\Dcal[\Acal^\gm(yf(x))]$ and $\EE[\Acal^\gm(yf(x))] \leq \err^\gm(f)$.
Therefore we have the following generalization error bound for MKL in the case of a convex combination of kernels.
\begin{theorem}\label{thm:main}
Fix $\gm > 0$ and $\dt \in (0,1)$.  Let $\Kcal = \{\kappa_1,\ldots, \kappa_p\}$ be a family of kernels containing $p$ base kernels and let $\z = \{z_i\}_{i=1}^{m}$ be a randomly generated sample from distribution $\Dcal$.  Then with probability $1-\dt$ over the random draws of samples of size $m$, every $f \in \Fcal_{\Kcal_{\mathrm{con}}}$ satisfies
\begin{eqnarray*}
\err(f) & \leq &  \hEE[\Acal^\gm(yf(x))] + \frac{1}{\gm}\LB \max_{1\leq j \leq p} \frac{2}{m}\sqrt{\sum_{i=1}^{m} \kappa_j(x_i,x_i)} + \sqrt{\frac{8\ln p}{m}}\RB\\
&& +\ 3\sqrt{\frac{\ln(2/\dt)}{2m}}\, . 
\end{eqnarray*}
Also, if each kernel $\kappa_j$ is normalised and bounded by $R^2 \geq \kappa_j(x,x)$ for all $x \in \Xcal$ and $j \in \{1,\ldots,p\}$, we have
\begin{eqnarray*}
\err(f) & \leq &  \hEE[\Acal^\gm(yf(x))] + \frac{1}{\gm}\LB \frac{2R}{\sqrt{m}}
 + \sqrt{\frac{8 \ln p}{m}}\RB + 3\sqrt{\frac{\ln(2/\dt)}{2m}}\, .
\end{eqnarray*}
\end{theorem}
\begin{proof}
Each kernel $\kappa_j$ defines the class
$
\Fcal_j = \{ x \mapsto \langle w, \phi_{\kappa_j}(x)\rangle : \|w\| \leq 1\}
$. 
Hence, applying Theorem \ref{thm:main_rad} to the class
$
\Acal^\gm(\Fcal_{\Kcal}) = \{ \Acal^\gm\circ f : f \in \Fcal_\Kcal\}
$, 
we have
\begin{eqnarray*}
\err(f) & \leq & \EE_\Dcal[\Acal^\gm(yf(x))] \\
&\leq&\hEE[\Acal^\gm(yf(x))] + \hat{R}_m(\Acal^\gm(\Fcal_{\Kcal_{\mathrm{con}}})) + 3 \sqrt{\frac{\ln (2/\dt)}{2m}} \\
&\leq&\hEE[\Acal^\gm(yf(x))] + \frac{1}{\gm}\hat{R}_m(\Fcal_{\Kcal_{\mathrm{con}}}) + 3 \sqrt{\frac{\ln (2/\dt)}{2m}} \\
&=&\hEE[\Acal^\gm(yf(x))] + \frac{1}{\gm}\hat{R}_m\left(\bigcup_{j=1}^p\Fcal_{\kappa_j}\right) + 3 \sqrt{\frac{\ln (2/\dt)}{2m}} \\
& \leq & \hEE[\Acal^\gm(yf(x))] + \frac{1}{\gm}\LB\max_{1 \leq j \leq p} \hat{R}_m(\Fcal_j) +\sqrt{\frac{8 \ln p}{m}}\RB  + 3 \sqrt{\frac{\ln (2/\dt)}{2m}} \\
& \leq & \hEE[\Acal^\gm(yf(x))] + \frac{1}{\gm}\LB\max_{1 \leq j \leq p} \frac{2}{m}\sqrt{\sum_{i=1}^{m} \kappa_j(x_i,x_i)} +\sqrt{\frac{8 \ln p}{m}}\RB  + 3 \sqrt{\frac{\ln (2/\dt)}{2m}} \\
& \leq & \hEE[\Acal^\gm(yf(x))] + \frac{1}{\gm}\LB\max_{1 \leq j \leq p} \frac{2}{m}\sqrt{mR^2} +\sqrt{\frac{8 \ln p}{m}}\RB  + 3 \sqrt{\frac{\ln (2/\dt)}{2m}} \\
& = & \hEE[\Acal^\gm(yf(x))] + \frac{1}{\gm}\LB\frac{2R}{\sqrt{m}} +\sqrt{\frac{8 \ln p}{m}}\RB  + 3 \sqrt{\frac{\ln (2/\dt)}{2m}}\, , \\
\end{eqnarray*}
where the third line comes from applying Theorem~\ref{thm:lip} with with Lipschitz constant $L = 1/\gm$. The forth line comes by applying Equation~\eqref{eq:con}. The fifth line comes by applying Theorem~\ref{thm:boost_rad}. The 6th line follows from Theorem~\ref{thm:kern_rad}. Finally, the 7th line follows from the hypothesis that $\kappa_j(x,x)\le R^2\ \forall x$.
\end{proof}

\section{Discussion}

Using the notation from above, the un-normalized version of the bound of Theorem 8 of~\cite{hs-11} is
\begin{eqnarray*}
\err(f) & \leq &  \herr^{\gm}(f) + \frac{2}{\gm m} \max_{1\leq j \leq p} \sqrt{\sum_{i=1}^{m} \kappa_j(x_i,x_i)}
 + 5 \sqrt{\frac{\ln((p+3)/\dt)}{2m}}.
\end{eqnarray*}
Comparing this to Theorem~\ref{thm:main} (the corrected version), we can see that the major difference is the fact that 
$1/\gm$ is multiplying $\ln p$ in Theorem~\ref{thm:main}, while it is not in the Theorem~8 of~\cite{hs-11}. However, the latter was obtained by \emph{incorrectly} assuming that $\hat{R}_m(\Acal^\gm(\mathrm{con}(\Fcal)))$ is upper bounded by $\hat{R}_m(\mathrm{con}(\Acal^\gm(\Fcal)))$. 

While Theorem 2 of \cite{hs-11} shows an additive dependence on the logarithm of the number of kernels it has an additional term that includes the number of kernels $d$ involved in the final solution and this number is also multiplied by the logarithm of the number of kernels. However, these quantities are separate from the main margin complexity term. A similar result could be obtained for the Rademacher bound given here resulting in a partial independence between the complexity and number of kernel terms, but with the final number of active kernels entering as an additional complexity term.


\section*{Acknowldgements}

We thank an anonymous reviewer from the Journal of Machine Learning Research who pointed out the flaw in the original proof of Theorem~\ref{thm:main}.  

%
%

\bibliography{references}
\bibliographystyle{abbrv}

\end{document}